\documentclass[10pt,reqno]{amsart}
\usepackage{graphicx}
\usepackage{indentfirst,csquotes}

\usepackage{hyperref}
\usepackage{url}
\usepackage{natbib}

\usepackage{parskip}

\usepackage{subcaption}
\usepackage{booktabs}

\usepackage{amssymb, amsthm, amsmath}

\newtheorem{theorem}{Theorem}

\newtheorem{definition}{Definition}

\newtheorem{lemma}{Lemma}
\newtheorem{assumption}{Assumption}

\hypersetup{ colorlinks=true, linkcolor=black, filecolor=black, urlcolor=black }

\usepackage{lipsum}

\begin{document}
\title[EAGLE: Entropy-Aware Object Localization]{EAGLE: Entropy-Aware Geometric Learning Enhancer for Partially Polygon-Supervised Object Localization} %%%%%%%%%%%%
\author[N Stewart]{Naveen Stewart}
\date{\today}
\address{}
\email{}
\maketitle

\begin{abstract}
Object detection is a critical application of computer vision, yet models trained uniformly on bounding boxes lack the geometric precision to be fully effective in inference settings, especially at high IoU thresholds. One common approach is the introduction of polygon-supervised localization, where a model is trained explicitly for localization but with access to polygon masks during training. Among such methods, partial polygon supervision is viable due to the higher annotation costs of developing fully supervised training datasets. However, this introduces a further challenge: in object classes of relatively low shape complexity, the addition of training masks and subsequent segmentation loss may obscure box localization, often worsening model performance despite access to deeper geometric information. To address the practical limitations of both full and partial polygon supervision, we introduce a mathematical framework for quantifying per-class geometric complexity by combining information theory and statistical shape analysis. With this framework, we develop Entropy-Aware Geometric Learning Enhancer (EAGLE), a lightweight, post-training model head that learns geometric information for the localization model, improving accuracy via edge refinement and localization rescoring using entropy-informed updates. Empirical analysis across MS COCO and Pascal SBD demonstrates the entropy framework meaningfully predicts shape reconstruction error and accuracy differences among mask-derived and standard box predictions, while the implemented EAGLE framework achieves consistent improvement on already SOTA detection frameworks (+0.9 mAP50-95 points on MS COCO) at minimal observed inference overhead.
\end{abstract}

\section{Introduction}

Computer vision tasks are increasingly demanding data-efficient methods for model training and deployment, especially with the transition towards widespread implementation \citep{khatib2026,zhang2026}. As a primary setup for localization, object detection has remained a predominant method for real-time industrial applications, and relies on bounding box annotations in standard supervised methods. In many instances, 2D object detection serves as the first step in a multi-stage model setup, often serving as the initial localization step and regional proposals to Vision Language Models (VLMs) \citep{li2022}, object tracking \citep{zhang2022}, robotics systems \citep{wang2026}, and as an initial pass for 3D computer vision applications.

Whereas typical object detection relies solely on bounding boxes, more recent methods have attempted to replicate the higher precision of instance segmentation via polygon supervision. Many detectors use polygon supervision during training and later discard the polygon head at inference; such methods give a model access to precise geometric information that boxes alone lack. While promising for avoiding the substantial overhead of instance segmentation, such training is heavily resource intensive: mask annotations require significantly more manual effort, often increasing annotation time by a near tenfold factor \citep{lin2015,bellver2019,cordts2016}. This unrealistic cost often necessitates partial polygon supervision, where all object instances are annotated with bounding boxes and only a portion are annotated with polygon masks.

However, for some classes, joint mask learning can interfere with box localization despite providing richer geometric information. Such drawbacks are common in low-complexity objects, where understanding precise features provides minimal localization gain and instead distracts the model. Methods for quantifying object shape complexity are further incomplete, as existing heuristics can not reliably predict complexity when exposed to masks, often using manual inspection and intuition rather than quantifiable methods.

To address these gaps, we introduce Entropy-Aware Geometric Learning Enhancer (EAGLE), a lightweight, post-training refiner that combines entropy-conditioned geometric proposals alongside proposal selection and postprocessing to refine object detection using learned geometric features without compromising native model performance. To provide the theoretical backing for the architecture, we build a mathematical framework combining information theory and statistical shape analysis to quantify shape uncertainty, and derive an upper bound on expected box IoU based on post-inference shape ambiguity. This framework motivates EAGLE's entropy-conditioned geometric refinement and training. Our main contributions are as follows:
\begin{itemize}
    \item We develop a mathematical framework for quantifying shape uncertainty in a partial polygon dataset. Theorems include an upper bound on continuous shape entropy $h(S | A_{box})$, or the uncertainty about object shape given only bounding box annotations, an explicit form for discrete entropy $H(S | A_{box})$, and an upper bound on model Intersection over Union (IoU) performance that is strictly decreasing with respect to the remaining shape uncertainty post-inference.
    \item The heuristic $H(S | A_{box})$ is validated on benchmark datasets as a predictor of shape reconstruction error and the relative gain of mask-derived box $\text{mAP}_{50-95}$ against the detection head. We show high correlation between the two factors, emphasizing the heuristic as a valuable form for measuring object uncertainty.
    \item We evaluate EAGLE across several benchmarking datasets and model frameworks, showcasing substantial $\text{mAP}_{50-95}$ gains at minimal inference cost. The head can be trained with a budget of polygon instances, rather than full polygon supervision, and we demonstrate effectiveness with only 40\% of polygon annotated objects.
\end{itemize}

\section{Related Work}

\textbf{Geometric Supervised Object Localization.} While early frameworks of instance segmentation \citep{hariharan2014,dai2016,he2018} relied heavily on full polygon supervision to retain geometric features and enhance mask predictions, recent work on semi-supervised annotations \citep{hu2023,zhou2020,berrada2024}, weak supervision \citep{tian2021,lan2021,cheng2023}, and partial supervision \citep{hu2018,kuo2019,wang2022} when applied to standard computer vision methods has demonstrated increasingly capable model architectures with lower annotation costs. Many frameworks have developed methods for utilizing various types of geometric information to enhance feature recognition for localization of the object class, given a variety of types of training data of varying annotation type \citep{chen2019,kirillov2020,lu2018,zhang2018}.

Notably, work from \citet{bhowmik2025} identifies that overlapping bounding box inferences often present partial geometric information about the underlying object, and aggregation of such inferences can significantly improve localization accuracy. While this approach prioritizes the geometric information from localization, our method targets polygon representations of the underlying object and trains models to interpret geometric features from partially supervised training data.

\textbf{Polygon and Mask Supervised Object Detection.} Various works have aimed to use polygon supervision to advance object detection, noting that standard multi-task frameworks present challenges based on different learning priorities between detection and segmentation \citep{jaegar2020,cao2019,brahmbhatt2017}. Notable examples of geometric feature interference include Mask R-CNN \citep{he2018}, where a joint training setup results in exact precision features (fine edges and contours) distracting the model from the spatial features more valuable to coarse bounding box alignment. Further frameworks such as You-Only-Look-Once (YOLO) \citep{jocher2026} reports lower box AP among instance segmentation models relative to identically trained object detection, implying that polygon supervision can have drawbacks despite higher geometric information.

To address these challenges, work from \citet{sun2026} approaches object detection via polygon-supervised localization guided by regression from polar representation. While the technique effectively improves accuracy, it substantially increases model inference time and is bottlenecked by the initial prediction of object center. In our approach, we use a more direct non-Euclidean based approach for box refinement rather than polar regression, aiming to avoid inference cost while still learning geometric features of observed objects. The method of weakly supervised instance segmentation from \citet{kweon2025} combines various classes of weak labels as a direct replacement for polygon masks, finding that some geometric information can be preserved without any mask annotations. Our approach shares the same underlying goal of reducing dataset annotation costs, but does so via an annotation basis of bounding boxes, then partially annotating the dataset with masks.

Note that while traditional polygon supervision in 2D object detection focuses on standard Euclidean geometry-derived interpretations of model features, our work represents a progression towards geometric deep learning \citep{li2025,feydy2020,bronstein2021,veljkovic2026,cao2020,hoang2025} and its applications to computer vision.

\textbf{Information Theory Applied to Deep Learning.} Information theory \citep{shannon1948, cover2006} is deeply implemented in deep learning frameworks \citep{tishby2015,gabrie2018,michael2018,jeon2022,peng2025}. Existing methods frequently classify entropy and mutual information with respect to inference channels and between layers in CNNs, MLPs, and sub-layers of transformers. While rate-distortion theory has been applied in part to annotations \citep{xu2021}, this is typically understood from a broader perspective of input information, without considering annotation type or in respect to a particular observed dataset. We resolve this via the application of statistical shape analysis integrated with information theory to derive entropy formulas.

\textbf{Statistical Shape Analysis.} Many foundational concepts in statistical shape analysis \citep{dryden2016, small1996} overlap with developments in machine learning, particularly methods for identifying dataset-level trends for physical objects in computer vision. Applying Kendall shape spaces developed in \citet{kendall1984} as part of Riemannian geometry often becomes a powerful tool for applications in VAEs \citep{vadgama2022}, probabilistic segmentation \citep{park2024}, and preserving data privacy \citep{soto2022,he2026}. In computer vision especially, Kendall shape spaces have been used at times for embedding images to isolate intrinsic shape. However, similar methods do not yet exist for annotation-level embedding, where bounding boxes and polygon representations do not have such frameworks. To address this gap, we formalize a mathematical framework for annotation-level embeddings for 2D localization frameworks.

\section{Preliminaries}

To begin the problem setup, let $S$ represent the variable of object shape for a dataset, and denote $h(S | A_{box})$ and $H(S | A_{box})$ as the per-instance bits of uncertainty on object shape, for differential and discrete entropy, respectively. The bounding box annotations across the dataset collectively form $A_{box}$. Write dataset $D_{box} = \{(S_1, A_1), ... , (S_N, A_N)\}$ where $(S_i, A_i)$ describes observed object $i$, where the dataset contains $N$ total object instances. Note that $D$ describes the instances of a single object class, and a set of images can be written as the union of datasets of individual object classes.

Polygon annotations are represented collectively as $A_{mask}$.

\begin{assumption}
    $H(S | A_{mask}) \approx 0$. For sufficient polygon vertices, a mask accurately outlines the true object shape, implying a sufficient representation of shape. Although manual annotations may be imperfectly drawn, masks are high precision representations of object shape, and as such are justifiably assumed to have approximately no shape entropy.
\end{assumption}

For all mask instances of the dataset, write $D_{mask} = \{(S_1, M_1), .., (S_N, M_N)\}$ corresponding to the same indices $i$ of $D_{box}$. Resample such that $M_i = \{(x_1, y_1), ... (x_K, y_K)\} : x_i, y_i \in \mathbb{R}^2$ for $K$ vertices. Methods such as arc-length parameterization or curvature-adaptive resampling are preferable and retain shape representation precisely.

\section{Conditional Shape Entropy}

We begin with a definition regarding Kendall shape spaces.

\begin{definition}
    Define $\Sigma_2^K = S^{2K-3} / SO(2)$ as the planar Kendall shape space of $K$ landmarks in 2-axes. $SO(2)$ denotes the special orthogonal group in 2-axes, and $S^{2K - 3}$ denotes the pre-shape sphere of $K$ landmarks. $S_i \in \Sigma_2^K$ for all $S_i$ from the training dataset. All elements of $\Sigma_2^K$ are invariant to translation, scale and rotation.
\end{definition}

This definition allows for the conversion of polygon masks to intrinsic shape objects. The exponent $2K - 3$ is the standard exponent for $K$ landmarks \citep{kendall1984}, derived from $2K$ coordinates, subtracting 2 for translation invariance and 1 for scale.

\begin{definition}
    Let $S_1, S_2$ be equivalence classes in $\Sigma_m^K$. True geodesic distance $d_G(S_1, S_2)$ defined with respect to shapes $S_1, S_2$ is provided by:
    \begin{equation}d_G (S_1, S_2) = \arccos \left( \max_{R \in SO(m)} \langle s_1, Rs_2 \rangle \right): s_1 \in S_1, s_2 \in S_2\end{equation}
    where vectors $s_1, s_2$ are chosen arbitrarily from their respective equivalence classes.
\end{definition}

\begin{definition}
    The Frechet Mean of shapes $S_1, ..., S_N \in \Sigma_2^K$ is given by $\mu = \arg \min_{S \in \Sigma_2^K} \sum_{k=1}^N d^2_G (S, S_k)$.
\end{definition}

We now wish to construct a single matrix of observed shapes, call this $V$, to compare relative shape similarities for estimating $S$ without access to $D_{mask}$. We do so by converting from observed shapes $\{S_i : S_i \in \Sigma_2^K\}_{i=1}^N$ and Frechet Mean $\mu$, under the following lemma.

\begin{lemma}
    Let $Log_\mu (S_i) = \frac{\theta_i}{\sin \theta_i} (R_iS_i - \mu \cos \theta_i)$ where $\theta_i = d_G (\mu, S_i)$ and $R_i = \arg \max_{X \in SO(2)} \langle S_i, X \mu \rangle$. Denote $T_\mu \Sigma_2^K = \{\gamma' (0) : \gamma (0) = \mu , \gamma : [-\epsilon, \epsilon] \to \Sigma_2^K\}$ as the Euclidean tangent space centered at $\mu$. The logarithmic mapping $Log_\mu : \Sigma_2^K \to T_\mu \Sigma_2^K$ preserves shape identity, true geodesic distance from the mean and locally approximates geometric relations.
\end{lemma}

These properties, as proven in \citet{small1996}, allow for the use of covariance analysis on matrix $V = \begin{bmatrix} v_1 & \cdots & v_N \end{bmatrix}^T \in \mathbb{R}^{N \times (2K - 4)}$ and $v_i = Log_\mu (S_i) \in \mathbb{R}^{2K - 4}$, where the $2K - 4$ vector size is from a property proven in \citet{kendall1984}.

For a given annotation $A_i = (x_c, y_c, w, h)$, consider $h(S_i | A_i)$. The Kendall shape space representation implies both translation and scale are irrelevant to the intrinsic object shape of the corresponding $S_i$. As such, disregarding $x_c$ and $y_c$ does not reduce mutual information on $S_i$. Further, as scale does not determine intrinsic object shape, the absolute scale of width and height are further irrelevant to object shape, leaving $\frac{w}{h}$ as the only relevant quantity derived from $A_i$ about $S_i$. Let $r_i = \log \frac{w}{h}$. As such $h(S_i | A_{i}) = h(S_i | r_i)$. For aspect ratio distribution $R$, we therefore have $h(S | A_{box}) = h(S | R)$. While this follows logically based on shape invariance, it is best to include as a modeling assumption, as statistical patterns may alter the direct relation, such as larger boxes weakly correlating to certain shapes, despite perfect invariance prohibiting such relations in the shape space. As such:
\begin{assumption}
    $H(S_i | A_i) = H(S_i | R_i)$.
\end{assumption}
Further, the properties preserved under Lemma 1 implies that $h(S | R) \approx h(V | R)$, as $V$ represents the shape distribution invariant to scale, translation and rotation. We now derive an upper bound on $h(V | R)$, and therefore an approximate upper bound on $H(S | R)$ and $h(S | A_{box})$ with the inequality:
\begin{equation}h(V | R = r) \le \frac{1}{2} \log_2 ((2 \pi e)^d \det Cov(V | R = r))\end{equation}
Note how $d = 2K - 4$. Writing $C(r) = Cov(V | R = r)$, we derive the following:

\begin{lemma}
    $C(r) = \frac{\sum_i \exp \left( - \frac{(r - r_i)^2}{2h^2} \right) (v_i - \mu_r) (v_i - \mu_r)^T}{\exp \left( - \frac{(r - r_i)^2}{2h^2} \right)}$ where $h = 0.9 \sigma N^{- \frac{1}{5}}$, pairwise aspect ratios are $\{r_i\}_{i=1}^N$ corresponding to masks $M_i$, and $\mu_r = \frac{\sum_i K_h (r - r_i) v_i}{\sum_i K_h (r - r_i)}$ for Gaussian Kernel $K_h$.
\end{lemma}

\begin{proof}
    Proven in Appendix A.
\end{proof}

Noting how $C(r) \in \mathbb{R}^{(2K - 4) \times (2K - 4)}$ is a covariance matrix, it is therefore positive semidefinite. To avoid the logarithmic operation being undefined, if $C(r)$ has an eigenvalue of 0, replacing $C'(r) = C(r) + \epsilon$ resolves this, as $C(r) = Q \Lambda Q^T$ for $\Lambda = \text{diag} (\lambda_1, ..., \lambda_{2K - 4})$. Because of this eigendecomposition, we know $\det C(r) = \prod_k \lambda_k$. By substitution we have $h(V | R = r) \le \frac{1}{2} \log_2 ((2\pi e)^{2K - 4} \prod_k \lambda_k) = \frac{1}{2} (2K - 4) \log_2 (2 \pi e) + \frac{1}{2} \sum_{k=1}^{2K - 4} \log_2 \lambda_k = \frac{1}{2} \sum_{k=1}^{2K - 4} \log_2 (2 \pi e \lambda_k)$.

Now have $h(V | R) = \int p(r) h(V | R = r) dr$ where the integral is approximated near-continuously based on observed values of $r$, and $p(r) = \frac{1}{N}$. With this, write $h(V | R) \le \frac{1}{N} \sum_{i=1}^N \frac{1}{2} \sum_{k=1}^{2K - 4} \log_2 (2 \pi e \lambda_{ik})$ where $\lambda_{ik}$ represents the $k$-th eigenvalues of the covariance matrix $C(r_i)$. This forms the first major contribution.

\begin{theorem}
    $h(S | A_{box}) = h(S | R) \approx h(V | R) \le \frac{1}{N} \sum_{i=1}^N \frac{1}{2} \sum_{k=1}^{2K - 4} \log_2 (2 \pi e \lambda_{ik})$. Further, if we assume a normal distribution $V | R = r \sim \mathcal{N} (\mu_r, C(r))$, this forms a more direct $h(S | A_{box}) \approx \frac{1}{N} \sum_{i=1}^N \frac{1}{2} \sum_{k=1}^{2K - 4} \log_2 (2 \pi e \lambda_{ik})$.
\end{theorem}

In finding discrete entropy, we discretize aspect ratios via binning through Knuth's Rule \citep{knuth2013} and cluster shape vectors via Tibshirani's Gap Statistic \citep{tibshirani2001}. For aspect ratio bins, let $M = \arg \max_A F(A)$ with:
\begin{equation}
    F(A) = N \ln (A) + \ln \Gamma \left( \frac{A}{2} \right) - A \ln \Gamma \left( \frac{1}{2} \right) - \ln \Gamma \left( N + \frac{A}{2} \right) + \sum_{k=1}^A \ln \Gamma \left( N_k + \frac{1}{2} \right)
\end{equation}
be the optimal number of bins. Note $N$ is the number of object instances, and $N_k$ is the number of aspect ratios assigned to the $k$-th bin during optimization. Now group aspect ratios into bins $\mathcal{B} = \{B_1, ..., B_M\}$.

For clustering shapes, we wish to find $C$ to optimally cluster vectors $v_1, ..., v_N$ into $C$ clusters. Using Tibshirani's Gap Statistics where $Gap(C) = \frac{1}{B} \sum_{b=1}^B \log (W_{C, b}') - \log (W_C)$ is maximized via choice of $C$ using $B = 100$. The value $W_{C, b}' = \sum_{c=1}^C \frac{1}{2N_c} \sum_{i \in \text{cluster c}} \sum_{j \in \text{cluster c}} ||x_i - x_j ||^2_2$ is taken where vectors $x_i, x_j$ are from the reference dataset, while $W_C = \sum_{c=1}^C \frac{1}{2N_C} \sum_{i \in \text{cluster c}} \sum_{j \in \text{cluster c}} ||v_i - v_j||^2_2$ where $N_c$ is per-cluster sample size. After $C$ is chosen, write shapes $\{Y_1, ..., Y_C\}$ where $Y_i = \frac{1}{N_i} \sum_{a \in \text{cluster i}} v_a$ for $N_i$ shapes in the cluster. Once both are discretized, we derive $H(S | A_{box})$ using the observed probabilities of a given aspect ratio being of each cluster. Write $H(S | A_{box}) = - \sum_{B_x} P(B_x) \sum_{Y_y} P(Y_y | B_x) \log_2 (P(Y_y | B_x))$. If $N_x$ is the number of annotations of aspect ratio bin $B_x$ and $N_{x, y}$ is the number of annotations of aspect ratio bin $B_x$ and cluster $Y_y$, it follows from the prior derivations of bins and clusters that:

\begin{theorem}
    $H(S | A_{box}) = - \sum_{B_x} \frac{N_x}{N} \sum_{Y_y} \frac{N_{x, y}}{N_x} \log \left( \frac{N_{x, y}}{N_x} \right)$
\end{theorem}

\section{Entropy to Localization Error}

Consider the following setup:

\begin{description}
    \item[$I$]: the image inputted through the model
    \item[$X$]: the latent shape of the target object
    \item[$X_B \in \mathcal{X}$]: the true bounding box of object $i$
    \item[$|\mathcal{X}|$]: the number of possible true bounding boxes
    \item[$\hat{B}_\theta = f_\theta (I)$]: the deterministic model bounding box output compared against bounding box $X_B$
    \item[$d_B(X_B, \hat{B}_\theta) = 1 - \text{IoU}_{box} (X_B, \hat{B}_\theta)$]
\end{description}

Let $Y = \arg \min_{Y_i \in \{Y_1, ..., Y_C\}} d_G (X, \text{Exp}_\mu (Y_i))$ be the cluster assigned to object shape $X$. Consider joint variable $Y_{B} = (X_B, Y)$. We first prove a Fano-type inequality for $H(Y | \hat{B}_\theta)$.

\begin{lemma}
    Under the conditions:

    \begin{enumerate}
        \item The alphabet $\mathcal{A}$ of $A$ satisfies $1 < | \mathcal{A} | < \infty$
        \item Functions $f, g$ are deterministic
        \item $d(a, b) \ge 0$ for all $a,b$
        \item There exists threshold $t > 0$ such that:
        \begin{equation} N_t = \sup_{\hat{b} \in \text{range}(g(B))} | \{a \in \mathcal{A} : d(f(a), \hat{b}) \le t\} | < | \mathcal{A} | \end{equation}
    \end{enumerate}
    
    we have: $H(A | B) \le \psi (\mathbb{E}[d(f(A), g(B))]) + \eta$ for strictly increasing function $\psi (z) = \left( \frac{\log_2 |\mathcal{A}| - \log_2 N_t}{t} \right) z$ and $\eta = 1 + \log_2 N_t$.
\end{lemma}

\begin{proof}
    Proven in Appendix C.
\end{proof}

Via Lemma 3, we can derive the following.

\begin{theorem}
    $\mathbb{E}[\text{IoU}_{box}(X_B, \hat{B}_\theta)] \le 1 - \psi^{-1}(H(Y | \hat{B}_\theta) - \eta)$ where $IoU_{box}$ represents the box IoU achieved by model $f_\theta (I)$. Qualitatively, this implies an upper bound on model performance that is strictly decreasing with respect to the remaining ambiguity about object geometry $Y$ given the prediction.
\end{theorem}

\begin{proof}
    Proven in Appendix C.
\end{proof}

Theorem 3 connects post-prediction shape ambiguity to a limit on box localization, implying entropy-informative predictions may yield stronger localization results. Combined with the framework for determining $H(S | A_{box})$, this provides a justifiable goal behind EAGLE: condition edge proposals and training weights to prioritize entropy-aware predictions via refinement.

\section{Entropy-Informed Learning}

Our EAGLE framework functions through entropy-inputted box proposals, selective refinement and rescoring, and postprocessing. Full descriptions of the framework are available in Appendix E.

\subsection{Geometric Head Training}

Head training uses all 118,287 MS COCO 2017 training images. Polygon supervision is only considered for 40\% of object instances, while 25\% of all polygon-supervised instances (or 10\% of all instances) are used in calculating discrete entropy.

Let a frozen object detection model produce box $b_i = (x_1, y_1, x_2, y_2)$, confidence score $s_i$, class $c_i$, and feature maps $F$. Let per-prediction entropy $H_i = 0.8 H_{c_i, r_i} + 0.2 H_{c_i}$, where $r_i$ is the aspect ratio bin of the prediction, $H_{c_i, r_i}$ is entropy given both bin and class, and $H_{c_i}$ is per-class entropy.

For each of 128 generated box candidates, a $9 \times 9$ feature grid is extracted, mapped to 48 channels via a series of convolutional layers. An MLP predicts attention logits at nine locations of each edge. For edge $e$, representation is defined as:
\begin{equation}
    a_{iek} = \frac{\exp f_{att} (v_{iek})}{\sum_{l=1}^9 \exp f_{att} (v_{iel})}, \qquad u_{ie} = \sum_1^9 a_{iek} v_{iek}
\end{equation}
The four edge vectors and confidence descriptors form a 246-dimensional input to a 192-dimensional channel. Entropy is used to condition this representation via the form:
\begin{equation}
    h_i = f_\theta ([u_{i1}, ..., u_{i4}, \hat{v}_i, g_i]) + A_\phi ([H_i, H_i^2, H_i g_i])
\end{equation}
The final projection is initialized to zero. A linear output predicts four distributions over 21 displacement bins $t_j \in [-1, 1]$. Their expectations together define edge refinements:
\begin{equation}
    p_{iej} = \text{softmax}_j (W_eh_i), \quad \delta_{ie} = \sum_j p_{iej}t_j, \quad b_i^p = b_i + 0.2 G(H_i) D_i \odot \delta_i
\end{equation}
where $D_i = (w_i, h_i, w_i, h_i)$ and $G(H) = 0.6 + 0.4 \sigma (4H - 2)$. Zero-initialized displacement logits produce zero refinement, as to avoid worsening an accurate prediction by the frozen model. Entropy determines normalized training weights proportional to $\text{clip} ( \exp (0.25 z_i), 0.65, 1.35)$ where $z_i$ represents batch entropy.

Training minimizes joint loss function:
\begin{equation}
    \mathcal{L} = \mathcal{L}_{loc} + 0.45 \mathcal{L}_{dist} + 0.1 \mathcal{L}_{margin} + 0.1 \mathcal{L}_{contact}
\end{equation}
The loss values combine CIoU and normalized coordinate error, interpolated displacement-bin cross-entropy, an IoU-improvement margin, and polygon-boundary attention supervision.

\subsection{Proposal Selector}

A separate selector learns from cached geometric-head outputs. For each of the selector candidates, the default action is the soft refinement:
\begin{equation}
    b_i^{(0)} = b_i + 0.75 \sigma \left( \frac{\Delta_i}{0.015} \right) (b_i^p - b_i)
\end{equation}
The selector may alter the default with the native box prediction or a refinement at 0.5, 0.75, 1, or 1.25 times the proposed displacement. It changes the default only when predicted utility minus the predicted harm exceeds a calibrated threshold. Otherwise, it keeps $b_i^{(0)}$ as the output, which may still be rescored.

\subsection{Postprocessing}

Various methods of lightweight refinement postprocessing are utilized and reported separately. Namely, entropy penalties, residual score calibration, geometric bias correction, duplicate damping, and refinement-strength adjustment as well as a bounded combination are utilized. Representative score transformations for entropy penalties and duplicate damping, respectively, are:
\begin{equation}
    s_i^E = s_i' e^{-\lambda_E H_i}, \qquad s_i^D = s_i' \exp \left[ -\lambda_D \left( \frac{(u_i - t)_+}{1 - t} \right)^2 (0.5 + 0.5 H_i) \right]
\end{equation}
where $u_i$ is the maximum overlap with higher confidence, same-class detection. Calibration adds a fitted feature-dependent correction to confidence logits. Geometric calibration corrects biases in the selected geometric head proposals. Strength adjustment aims to adjust the scale of edge transformations between the geometric head and predicted boxes, rather than just which edges are moved in which directions. Coefficients and operating points are determined prior to evaluation.

\section{Experimentation}

\subsection{Held-out Shape-reconstruction Error}

\textbf{Does entropy predict shape-reconstruction error?} We analyze correlation between shape entropy and modal and true-mask shape reconstruction error. The first of these models represents modal-cluster predictions, evaluated on a per-class basis. For datasets $D_{train}$ and $D_{val}$ respectively, both annotated with polygons of $K$ vertices, using the methodology for clustering and binning previously described, form shape clusters $(Y_1, ..., Y_C)$ and aspect ratio bins $(B_1, ..., B_M)$. Let the modal-cluster model given observed aspect ratio bin $B$ from dataset $D_{val}$ be $\hat{p}(B) = \arg \max_y \hat{P}_{train} (Y = y | R \in B)$. This model outputs the shape cluster with the highest proportion of training instances in aspect ratio bin $B$, once fitted on $D_{train}$ probabilities. Let $N_{val}$ be the number of instances evaluated in $D_{val}$. Define error $P_e = \frac{1}{N_{val}} \sum_{i=1}^{N_{val}} \textbf{1} \{ Y_i \ne \hat{Y}_i\}$ for model prediction $\hat{Y}_i$.

Similarly for shape reconstruction, let $\hat{p} (x)$ have the same definition. Then $\hat{M} = \arg \max_{m} \text{IoU}_{mask} (m, \hat{p}(x))$ where $m$ is taken across all polygon masks of $D_{train}$. For ground truth polygon $M$, define loss as $L = 1 - \frac{1}{N_{val}} \sum_{i=1}^{N_{val}} \text{IoU}_{mask} (M_i, \hat{M}_i)$.

We evaluate both $P_e$ and $L$ for each MS COCO 2017 \citep{lin2015} class with $K = 24$, evaluating correlation with $H(S | A_{box})$ calculated across the training dataset.

Results are shown in Figures \ref{figure1} and \ref{figure2}. We determine a very strong correlation between modal cluster loss ($r = 0.921$, $\rho = 0.886$, and permutation $p = 7.56 \cdot 10^{-33}$) and a strong, although slightly weaker, correlation with mask reconstruction ($r = 0.732$, $\rho = 0.803$, and permutation $p = 2.66 \cdot 10^{-14}$). Such correlations demonstrate promising correspondence between shape complexity and the derived heuristic, across both modal shape predictions and direct mask reconstruction error.

\begin{figure}[h]
    \centering
    \begin{minipage}{0.46\textwidth}
        \centering
        \includegraphics[width=\textwidth]{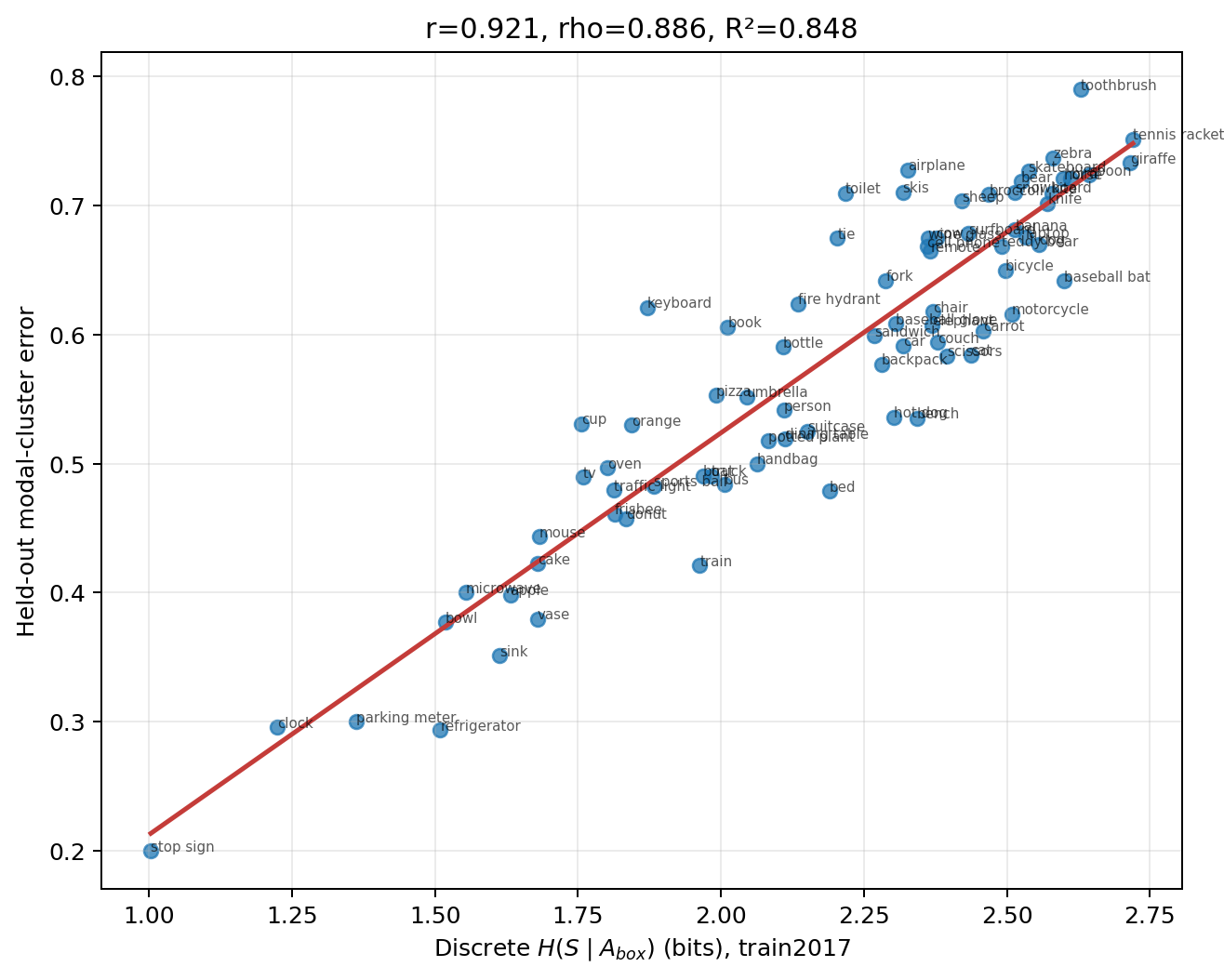}
        \caption{Modal cluster loss correlation with entropy: $r = 0.921$, $\rho = 0.886$, $R^2 = 0.848$.}
        \label{figure1}
    \end{minipage}
    \hfill
    \begin{minipage}{0.46\textwidth}
        \centering
        \includegraphics[width=\textwidth]{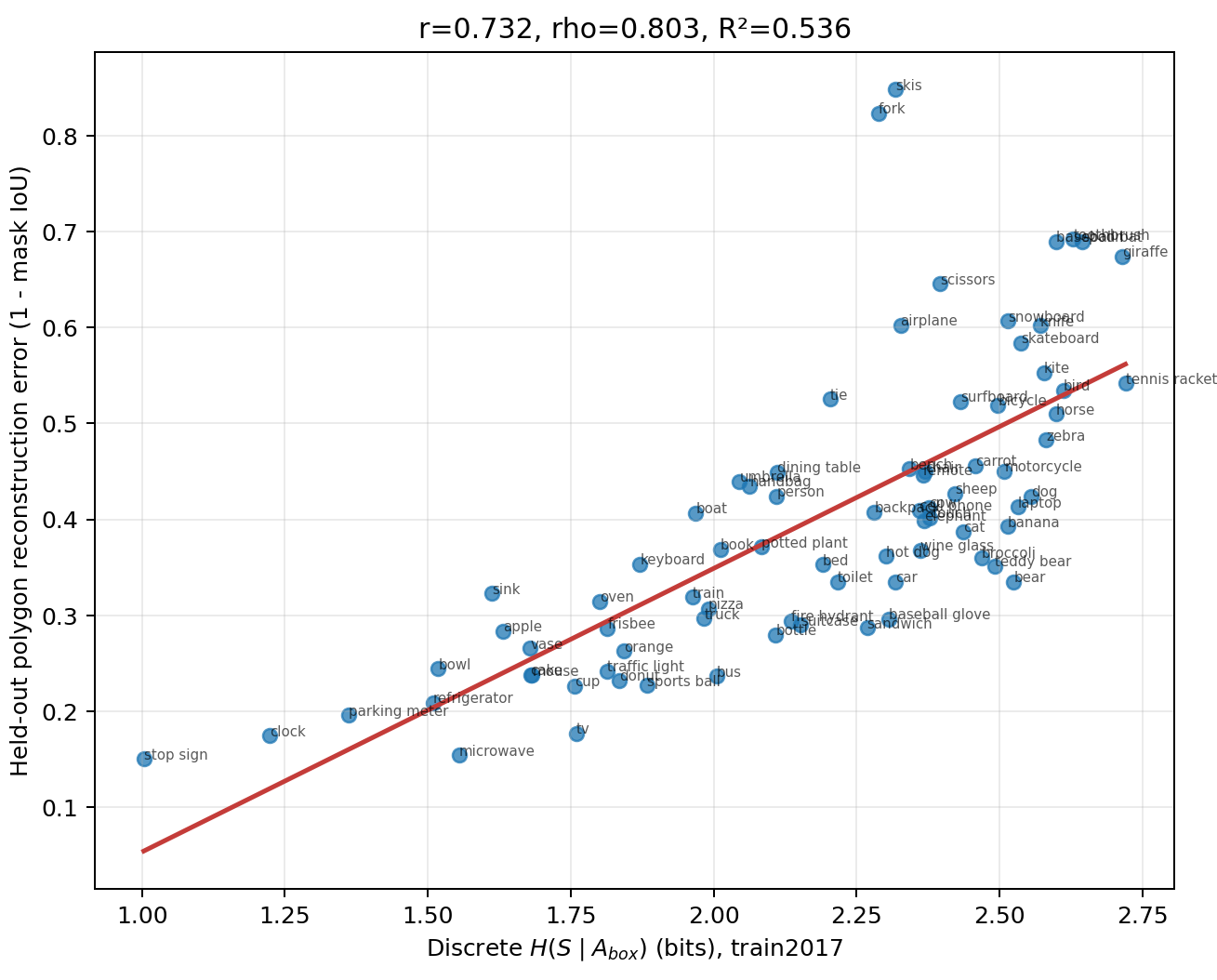}
        \caption{Mask reconstruction loss correlation with entropy: $r = 0.732$, $\rho = 0.803$, $R^2 = 0.536$.}
        \label{figure2}
    \end{minipage}
\end{figure}

\subsection{Mask-Derived Boxes and Detection Head Comparison}

\textbf{Does shape entropy predict the localization advantage of mask-derived boxes over native boxes for instance segmentation?} For class $c$, let $\Delta_c = AP_{mask} - AP_{box}$ where $AP_{mask}$ is derived from full polygon inference after converting polygons to their respective standard bounding boxes when comparing against ground truth boxes. All predictions come from the same instance segmentation model. We evaluate pretrained YOLO26n, s, m, l, and x \citep{jocher2026} respectively on the 20 Pascal SBD \citep{hariharan2011} classes, selected due to being out-of-distribution from the training images of YOLO26, the predominant one-stage CNN detector framework, as well as being well balanced for a dataset of its scale. We analyze correlation with $H(S | A_{box})$ evaluated per-class across the Pascal SBD training data. Outlying object classes are excluded from the analysis.

Results show that all YOLO26 variants had positive Spearman rank correlation with discrete entropy, with values $\rho = 0.553, 0.256, 0.461, 0.186$ and $0.319$ for n, s, m, l and x, respectively. All models had at most one outlying object class. When averaging $\Delta_c$ between models, no outliers remained and the resulting $\rho = 0.298$ signified positive correlation. However, when excluding the lowest correlation model, YOLO26l, from the average, one outlying class was formed and Spearman rho increased to $\rho = 0.540$. This correlation relies on excluding outliers, however still represents a promising result for the predictive nature of $H( S | A_{box})$.

\subsection{EAGLE-Enhanced Model Evaluation}

\textbf{How does EAGLE enhancement affect accuracy?} Using MS COCO 2017, we train three respective EAGLE enhancers using a 40\% polygon supervision budget for the refiner head. One refiner was trained for each of the YOLO26n, YOLO26m, and RT-DETRv2-S with ResNet-18 \citep{lv2024} frameworks, with the objective of robust evaluation across different computational cost CNN frameworks (YOLO26n and YOLO26m) alongside a transformer-based detector. The training methodology is described in Section 6.

As shown in Table \ref{sample-table}, the enhancer increases $\text{mAP}_{50-95}$ by roughly 0.64 points on YOLO26n, 0.27 on YOLO26m, and 0.06 on RT-DETRv2-S. When combined with the third phase of enhancer postprocessing, these values increase to 0.9, 0.45, and 0.16, respectively. Significant increases in $\text{mAP}_{50}$ are also shown, where the optimal bounded combination produces gains of 1.4, 0.8, and 0.29 respectively. A parameter count increase of 300,000 is observed uniformly across all methods and frameworks, which translates to a minimum and often negligible latency increase bounded by 0.1 milliseconds across all models. Most significantly, while the native YOLO26n instance segmentation box AP decreases by 0.5 points relative to standard object detection, our method provides a 0.9 point increase in favor of the EAGLE framework, effectively surpassing the prior polygon-supervised model by roughly 1.4 $\text{mAP}_{50-95}$ points after postprocessing.

Note that every other method combined with the EAGLE enhancer produced scores exceeding the baseline, and the reported variations are those with the most notable improvement.

\clearpage

\begin{table}[h]
\caption{Notable model performance for detection enhancement.}
\label{sample-table}
\begin{center}
\begin{tabular}{ll|lll|llll}
\toprule
Detector & Method & AP & $\text{mAP}_{50}$ & $\text{mAP}_{75}$ & Params (M) & Latency (ms)\\
\midrule
& Baseline & 40.25 & 55.80 & 43.65 & 2.4 & 2.3\\
\cmidrule{2-7}
YOLO26n& Enhancer & 40.89 & 56.77 & 44.17 & 2.7 & 2.3\\
\cmidrule{2-7}
& Duplicate Damping & 41.02 & 57.02 & \textbf{44.30} & 2.7 & 2.4\\
\cmidrule{2-7}
& Bounded Combination & \textbf{41.11} & \textbf{57.19} & 44.25 & 2.7 & 2.3\\
\midrule
& Baseline & 52.49 & 69.69 & 57.25 & 20.4 & 6.4\\
\cmidrule{2-7}
YOLO26m& Enhancer & 52.76 & 70.07 & 57.50 & 20.7 & 6.4\\
\cmidrule{2-7}
& Duplicate Damping & 52.88 & 70.28 & \textbf{57.64} & 20.7 & 6.4\\
\cmidrule{2-7}
& Bounded Combination & \textbf{52.94} & \textbf{70.48} & 57.56 & 20.7 & 6.4\\
\midrule
& Baseline & 48.10 & 65.07 & 52.09 & 20.0 & 6.8\\
\cmidrule{2-7}
RT-DETR& Enhancer & 48.16 & 65.17 & 52.17 & 20.3 & 6.9\\
\cmidrule{2-7}
& Duplicate Damping & 48.25 & \textbf{65.36} & 52.24 & 20.3 & 6.8\\
\cmidrule{2-7}
& Bounded Combination & \textbf{48.26} & \textbf{65.36} & \textbf{52.26} & 20.3 & 6.9\\
\bottomrule
\end{tabular}
\end{center}
\end{table}

\section{Conclusion}

In our work, we present a mathematical framework for quantifying shape entropy via non-Euclidean embedding in a Kendall Shape Space. We employ information-theoretic strategies and the derivation of a neighborhood-based Fano-type inequality in connecting localization error and Intersection over Union to discrete shape uncertainty. With this mathematical framework, we demonstrate empirically that post training development via an Entropy-aware Geometric Learning Enhancer, or EAGLE, is able to consistently teach model geometric features to a baseline localization model, even across only partially polygon-supervised annotations, significantly improving accuracy above SOTA frameworks. Alongside the most significant mAP increase with YOLO26n, positive accuracy gains across the proposed framework are shown to generalize to higher computational complexity CNN and transformer frameworks.

Further empirical analysis across benchmark datasets supports the validity of $H(S | A_{box})$ as a heuristic, showing significant correlation between shape reconstruction error and the proposed metric, alongside weaker but still relevant correlation with mask-derived and native box AP differences. Together, we present a combined mathematical and empirically-backed narrative regarding 2D localization and geometric feature learning: while polygon-supervised localization models that directly learn precise features alongside broader localization often suffer from interference, this interference and shape entropy can be measured directly via dataset-level heuristics. That same entropy can be utilized for targeted geometric learning via lightweight auxiliary heads, and the resulting accuracy significantly improves baseline SOTA learning models, with minimal inference overhead due to low parameter refinement. Together, refinement methods are applicable in data-limited settings, and function on partially polygon-supervised tasks.

\bibliography{references}
\bibliographystyle{bibstyle}

\appendix

\section{Proof of Lemma 2}

In forming Theorem 1, we rely on the following lemma:

\begingroup

\renewcommand{\thelemma}{2}
\begin{lemma}
    $C(r) = \frac{\sum_i \exp \left( - \frac{(r - r_i)^2}{2h^2} \right) (v_i - \mu_r) (v_i - \mu_r)^T}{\exp \left( - \frac{(r - r_i)^2}{2h^2} \right)}$ where $h = 0.9 \sigma N^{- \frac{1}{5}}$, pairwise aspect ratios are $\{r_i\}_{i=1}^N$ corresponding to masks $M_i$, and $\mu_r = \frac{\sum_i K_h (r - r_i) v_i}{\sum_i K_h (r - r_i)}$ for Gaussian Kernel $K_h$.
\end{lemma}

\endgroup

\begin{proof}
    First write:
    \begin{equation}C(r) = Cov( V | R = r) = \mathbb{E}[(V - \mathbb{E}[V | R = r]) (V - \mathbb{E}[V | R = r])^T | R = r]\end{equation}
    Under $\mu_r = \mathbb{E}[V | R = r]$, rewrite as:
    \begin{equation}C(r) = \mathbb{E} [ (V - \mu_r) (V - \mu_r)^T | R = r]\end{equation}

    Have $V' = V - \mu_r$ such that:
    \begin{equation}C(r) = \mathbb{E} [(V')(V')^T | R = r]\end{equation}
    \begin{equation}= \frac{\sum_i K_h (r - r_i) (v_i - \mu_r) (v_i - \mu_r)^T}{\sum_i K_h (r - r_i)}\end{equation}
    when using Gaussian Kernel $K_h$ for a smooth distribution of aspect ratios. It further implies:
    \begin{equation}\mu_r = \frac{\sum_i K_h (r - r_i) v_i}{\sum_i K_h (r - r_i)}\end{equation}
    Write:
    \begin{equation}K_h (r - r_i) = \frac{1}{\sqrt{2 \pi}h} \exp \left( - \frac{(r - r_i)^2}{2h^2} \right)\end{equation}
    Using $h = 0.9 \sigma N^{-\frac{1}{5}}$ as the bandwidth based on the formula established in \citet{silverman1986}. The value $\sigma$ represents the standard deviation of the distribution of aspect ratios. More computational complex methods of estimation across other bandwidth calculations may be used in place of this. From here, write that:
    \begin{equation}C(r) = \frac{\sum_i \frac{1}{\sqrt{2 \pi}h} \exp \left( - \frac{(r - r_i)^2}{2h^2} \right) (v_i - \mu_r) (v_i - \mu_r)^T}{\sum_i \frac{1}{\sqrt{2 \pi}h} \exp \left( - \frac{(r - r_i)^2}{2h^2} \right)} = \frac{\sum_i \exp \left( - \frac{(r - r_i)^2}{2h^2} \right) (v_i - \mu_r) (v_i - \mu_r)^T}{\sum_i \exp \left( - \frac{(r - r_i)^2}{2h^2} \right)}\end{equation}
\end{proof}

\section{Experimental Methodology}

\subsection{Shape Reconstruction Correlation Specifications}

The reconstruction models were fitted on 3,000 instances from MS COCO 2017 training data per-class to prevent non-uniform fitting. The Intel Xeon CPU was used for the experimentation. Per-category sampling random seed was calculated as 17 + COCO ID where the ID is the numbered COCO class of integers $0, 1, ..., 79$.

\subsection{Mask-Derived Boxes and Detection Head Comparison Specifications}

For the mask-derived and detection head comparison, we used the standard Pascal SBD training split to calculate per-class $H(S | A_{box})$, while the validation split was used for model evaluation. Training masks were resampled to 32 vertices, and up to 1500 instances per class were used. Clustering and binning were done via the described methodology including Knuth's rule for binning and Tibshirani's gap statistic. The random seed used was 17.

We evaluated the corresponding native box and predicted mask pair using the same confidence threshold and class prediction. Inferences used 640-pixel images and confidence threshold 0.001. Outlier sensitivity used the standard 1.5-IQR statistical definition. The GPU used for evaluating model inference was the NVIDIA RTX PRO 6000 Blackwell Server Edition.

\subsection{Training Hyperparameters for EAGLE}

For all of the following described training, a GPU identical to that for the mask-derived box and object detection evaluation was used. The selected random seed was 30 for all evaluations.

The full 118,287 images of MS COCO 2017 were used during training. All 5000 validation images of the same dataset were further used across all models to evaluate against the baseline. Image size was standardized at 640 by 640 pixels. The batch size used was 128.

Each image received 128 refined proposals, with ROI grid of 9 by 9. Training consisted of 60 total epochs, of which the first 25 were geometry, the next 20 were quality, and the last 15 were joint training as described in Appendix E for epoch scheduling. Model progress is evaluated at epochs 10, 20, ..., 60 and the highest performing epoch is saved for use with the selector and postprocessing. The AdamW \citep{loshchilov2018} optimizer is used across training with initial learning rate $7 \times 10^{-4}$ and weight decay $10^{-4}$. The head was trained with 2 warmup epochs and cosine decay converts learning rate to 0.03 that of the initial.

The detector remains frozen but runs inference on each training batch at each epoch. An HSV augmentation of (0.015, 0.7, 0.4) and horizontal flip with probability 0.5 was applied to each image of the training batch. The training used exponential moving average with decay $0.9999$ and $\tau = 2000$.

Only 40\% of training polygon instances were considered, and the rest were discarded and converted to bounding boxes. To fit entropy, 25\% of these polygon instances were separated and used.

For the selector, an AdamW learning rate of $4 \times 10^{-4}$ was used with 4,096 candidate predictions analyzed per optimizer step and 6 epochs. Weight decay was $10^{-3}$ with cosine learning rate decay.

For postprocessing, there were 8,000 fitting, 1,000 selection, and 2,000 calibration images. Score-calibration coefficients were fitted across six epochs with Adam at learning rate 0.01 and L2 coefficient of 0.001. Entropy-penalty strengths are $\{0.05, 0.15, 0.35, 0.7\}$, while score-calibration and geometric-bias strengths are $\{0.25, 0.5, 1\}$. The duplicate damping strengths were $\{0.25, 0.75\}$ at overlap thresholds of $\{0.7, 0.85\}$. The refinement strength adjustments are chosen from $\{-1, -0.5, 0.5, 1\}$. Calibrations are locked across the calibration images and then full validation was performed on MS COCO 2017 validation dataset.

\section{Proof of Localization Theorem}

We fully prove the following lemma used in Section 5 for the localization theorem:

\begingroup
\renewcommand{\thelemma}{3}
\begin{lemma}
    Under the conditions:

    \begin{enumerate}
        \item The alphabet $\mathcal{A}$ of $A$ satisfies $1 < | \mathcal{A} | < \infty$
        \item Functions $f, g$ are deterministic
        \item $d(a, b) \ge 0$ for all $a,b$
        \item There exists threshold $t > 0$ such that:
        \begin{equation} N_t = \sup_{\hat{b} \in \text{range}(g(B))} | \{a \in \mathcal{A} : d(f(a), \hat{b}) \le t\} | < | \mathcal{A} | \end{equation}
    \end{enumerate}
    
    we have: $H(A | B) \le \psi (\mathbb{E}[d(f(A), g(B))]) + \eta$ for strictly increasing function $\psi (z) = \left( \frac{\log_2 |\mathcal{A}| - \log_2 N_t}{t} \right) z$ and $\eta = 1 + \log_2 N_t$.
\end{lemma}
\endgroup

\begin{proof}
    Let $E$ be the error event $E = \textbf{1} \{d(f(A), g(B)) > t\}$. Write $H(A, E | B)$ as follows:
    \begin{equation}
        H(A, E | B) = H(A | B) + H(E | A, B) = H(E | B) + H(A | E, B)
    \end{equation}
    As $f, g$ are deterministic, error can be determined uniquely from information of both $A$ and $B$, implying $H(E | A, B) = 0$. As $H(E | B) \le H(E)$, we have:
    \begin{equation}
        H(A | B) + 0 \le H(E) + H(A | E, B)
    \end{equation}
    For simplicity let $D = d(f(A), g(B))$ and let probability $P = \mathbb{P} (E = 1) = \mathbb{P} (D > t)$. For binary event $E$, $H(E) = H_b (P) \le 1$. Expand $H(A | E, B)$ through cases on the possible values of $E$ with:
    \begin{equation}
        H(A | E, B) = P H(A | E = 1, B) + (1 - P) H(A | E = 0, B)
    \end{equation}
    If $E = 1$, the entropy $H(A | E = 1, B)$ is at most the base 2 logarithm of the number of elements in the space written as $\log_2 | \mathcal{A} |$, given that this value is discrete. If $E = 0$, the entropy is at most the base 2 logarithm of all possible elements in the neighborhood within $t$ distance of $g(B)$, implying $H(A | E = 0, B) \le \log_2 N_t$. Suppose $N_t \ge 1$ for now such that the operation $N_t$ is defined. Substitute with:
    \begin{equation}
        H(A | B) \le 1 + P \log_2 |\mathcal{A}| + (1 - P) \log_2 N_t
    \end{equation}
    \begin{equation}
        H(A | B) \le 1 + \log_2 N_t + P(\log_2 |\mathcal{A}| - \log_2 N_t)
    \end{equation}

    Based on Markov's Inequality with $P = \mathbb{P} (D >t) \le \frac{\mathbb{E}[D]}{t}$, write:

    \[H(A | B) \le 1 + \log_2 N_t + \left( \frac{\log_2 |\mathcal{A}| - \log_2 N_t}{t} \right) \mathbb{E}[D]\]

    Writing $\psi(z) = \left( \frac{\log_2 | \mathcal{A} | - \log_2 N_t}{t}z \right)$ and $\eta = 1 + \log_2 N_t$ gives:
    \begin{equation}
        H(A | B) \le \psi(\mathbb{E}[D]) + \eta
    \end{equation}
    \begin{equation}
        H(A | B) \le \psi(\mathbb{E}[d(f(A), g(B))]) + \eta
    \end{equation}
    Note how $\psi(z)$ is strictly increasing as $\frac{\log_2 |\mathcal{A}| - \log_2 N_t}{t}$ is a positive constant based on $N_t < | \mathcal{A} |$ guaranteed based on the fourth condition of the lemma. This concludes the main case $N_t > 0$.

    For the edge case $N_t = 0$, this implies $P = 1$ as no element is within the neighborhood of distance $t$ from $g(B)$, which implies:
    \begin{equation}
        H(A | B) \le 1 + P \log_2 | \mathcal{A}|
    \end{equation}
    \begin{equation}
        H(A | B) \le 1 + \log_2 |\mathcal{A}| \frac{\mathbb{E} [d(f(A), g(B))]}{t}
    \end{equation}
    The RHS is strictly increasing with respect to $\mathbb{E} [ d(f(A), g(B))]$, which satisfies the lemma when replacing $\psi (z) = \frac{\log_2 | \mathcal{A} |}{t}z$ and $\eta = 1$.
\end{proof}

Using this we now prove the theorem:

\begingroup
\renewcommand{\thetheorem}{3}
\begin{theorem}
    $\mathbb{E}[\text{IoU}_{box}(X_B, \hat{B}_\theta)] \le 1 - \psi^{-1}(H(Y | \hat{B}_\theta) - \eta)$ where $IoU_{box}$ represents the box IoU achieved by model $f_\theta (I)$. Qualitatively, this implies an upper bound on model performance that is strictly decreasing with respect to the remaining ambiguity about object geometry $Y$ given the prediction.
\end{theorem}

\begin{proof}
    We start by showing that the intended substitution into Lemma 3 satisfies its conditions, implying $H(Y_B | \hat{B}_\theta) \le \psi ( \mathbb{E}[ d_B (X_B, \hat{B}_\theta)]) + \eta$. Setting $A = Y_B$, we know that the alphabet size $| \mathcal{A} |$ is at most the upper bound on the number of potential states of $X_B$, multiplied by the number of potential states of $Y$. For some maximum image size $w$ by $h$ pixels, the former quantity is $w^2h^2$. As $Y$ is chosen from $C$ discretized shapes, any observed shape has at most $C$ possibilities corresponding to $X_B$. Therefore, $| \mathcal{A} | \le Cw^2h^2$ and is finite. Showing $| \mathcal{A} | > 1$ requires assumption.
    \begin{assumption}
        There are at least 2 distinct objects $X_B$, which implies $| \mathcal{A} | > 1$.
    \end{assumption}
    This is an extremely mild assumption, as any localization task conditioned on more than one possible bounding box input or output satisfies it. Together, the upper bound and assumption give $1 < | \mathcal{A} | < \infty$.

    Set $B = \hat{B}_\theta$. We further wish to define $f, g$ to be deterministic functions given $Y_B, \hat{B}_\theta$ for the use of Lemma 3. Setting $f(Y_B) = X_B$ and $g(\hat{B}_\theta) = \hat{B}_\theta$ achieves this condition. The condition $d(a, b) \ge 0$ where $d_B (X_B, \hat{B}_\theta) = 1 - \text{IoU}_{box} (X_B, \hat{B}_\theta )$ follows as $\text{IoU}_{box} (X_B, \hat{B}_\theta ) \in [0, 1]$, implying $d_B (X_B, \hat{B}_\theta ) = 1 - \text{IoU}_{box} (X_B, \hat{B}_\theta ) \in [0, 1]$.

    The last condition on the existence of $t > 0$ such that $N_t = \sup_{\hat{b} \in range(g(B))} |\{a \in \mathcal{A} : d(f(a), \hat{b}) \le t\} | < | \mathcal{A} |$ follows from geometric properties of Jaccard Distance. With the previous assumption, we assumed there are at least 2 distinct realized bounding boxes. Let $\delta = \min_{Y_B^i, Y_B^j \in \mathcal{A}, f(Y_B^i) \ne f(Y_B^j)} d_B (f(Y_B^i), f(Y_B^j))$ represent the minimum Jaccard Distance between any two distinct realized bounding boxes. Choose $t \in (0, \frac{\delta}{2})$. For all bounding box predictions $\hat{B}_\theta$, we have with use of the triangle inequality for Jaccard Distance that $d_B (f(Y_B^i), f(Y_B^j)) = \delta \le d_B (f(Y_B^i), \hat{B}_\theta) + d_B (f(Y_B^j), \hat{B}_\theta)$. Suppose both $Y_B^i$ and $Y_B^j$ were in neighborhood $N_t$. This implies $d_B (f(Y_B^i), \hat{B}_\theta), d_B (f(Y_B^j), \hat{B}_\theta) \le t < \frac{\delta}{2}$. It would then follow that $\delta \le d_B (f(Y_B^i), \hat{B}_\theta) + d_B (f(Y_B^j), \hat{B}_\theta) \le 2t < \delta$. This is a contradiction, implying at least one of $Y_B^i$ and $Y_B^j$ must be outside of the neighborhood $\{a \in \mathcal{A} : d(f(a), \hat{B}_\theta) \le t\}\}$, meaning $N_t < | \mathcal{A} |$ as the neighborhood is missing at least one element from the alphabet. As this same construction applies regardless of $\hat{B}_\theta$, we have $N_t < | \mathcal{A} |$ which satisfies the last condition.

    By substitution, we have:
    \begin{equation}
        H(Y_B | \hat{B}_\theta ) \le \psi (\mathbb{E}[d_B(f(Y_B), g(\hat{B}_\theta))]) + \eta = \psi (\mathbb{E}[d_B (X_B, \hat{B}_\theta )]) + \eta
    \end{equation}
    As $Y_B = (X_B, Y)$ is a joint variable including $Y$, it follows that $H(Y | \hat{B}_\theta) \le H(Y_B | \hat{B}_\theta)$. We write:
    \begin{equation}
        H(Y | \hat{B}_\theta) \le H(Y_B | \hat{B}_\theta ) \le \psi (\mathbb{E}[d_B (X_B, \hat{B}_\theta )]) + \eta
    \end{equation}
    \begin{equation}
        H(Y | \hat{B}_\theta ) - \eta \le \psi (\mathbb{E}[d_B (X_B, \hat{B}_\theta )])
    \end{equation}
    \begin{equation}
        \begin{aligned}
        \psi^{-1} (H(Y | \hat{B}_\theta ) - \eta) \le \mathbb{E}[d_B(X_B, \hat{B}_\theta)] \\
        = \mathbb{E}[1 - \text{IoU}_{box} (X_B, \hat{B}_\theta)] \\
        = 1 - \mathbb{E}[ \text{IoU}_{box} (X_B, \hat{B}_\theta)]
        \end{aligned}
    \end{equation}
    Rearranging Equation 31 produces the following:
    \begin{equation}
        \mathbb{E}[\text{IoU}_{box} (X_B, \hat{B}_\theta)] \le 1 - \psi^{-1} (H(Y | \hat{B}_\theta) - \eta)
    \end{equation}
    Note that as $\psi$ is strictly increasing, so is its inverse $\psi^{-1}$, and therefore the RHS is strictly decreasing with respect to $H(Y | \hat{B}_\theta)$ for some fixed $\eta$.
\end{proof}
\endgroup

\section{Expanded Empirical Results}

\subsection{Mask-derived and Native Box Correlation}

The following figures represent per-model correlation with $H(S | A_{box})$.

\begin{figure}[h]
\begin{subfigure}[t]{0.3\textwidth}
    \centering
    \includegraphics[width=\textwidth]{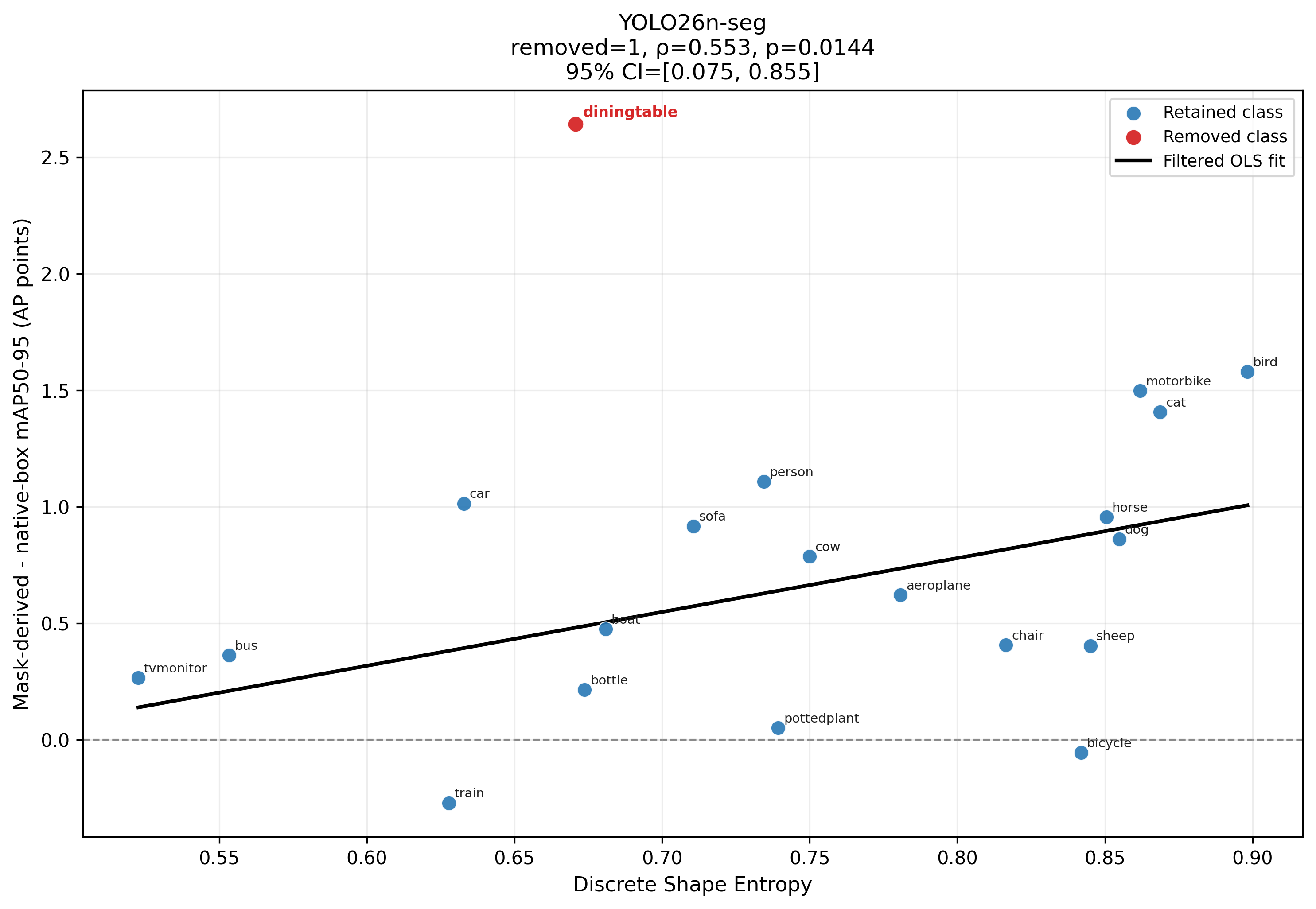}
    \caption{Correlation for YOLO26n}
\end{subfigure}
\hfill
\begin{subfigure}[t]{0.3\textwidth}
    \centering
    \includegraphics[width=\textwidth]{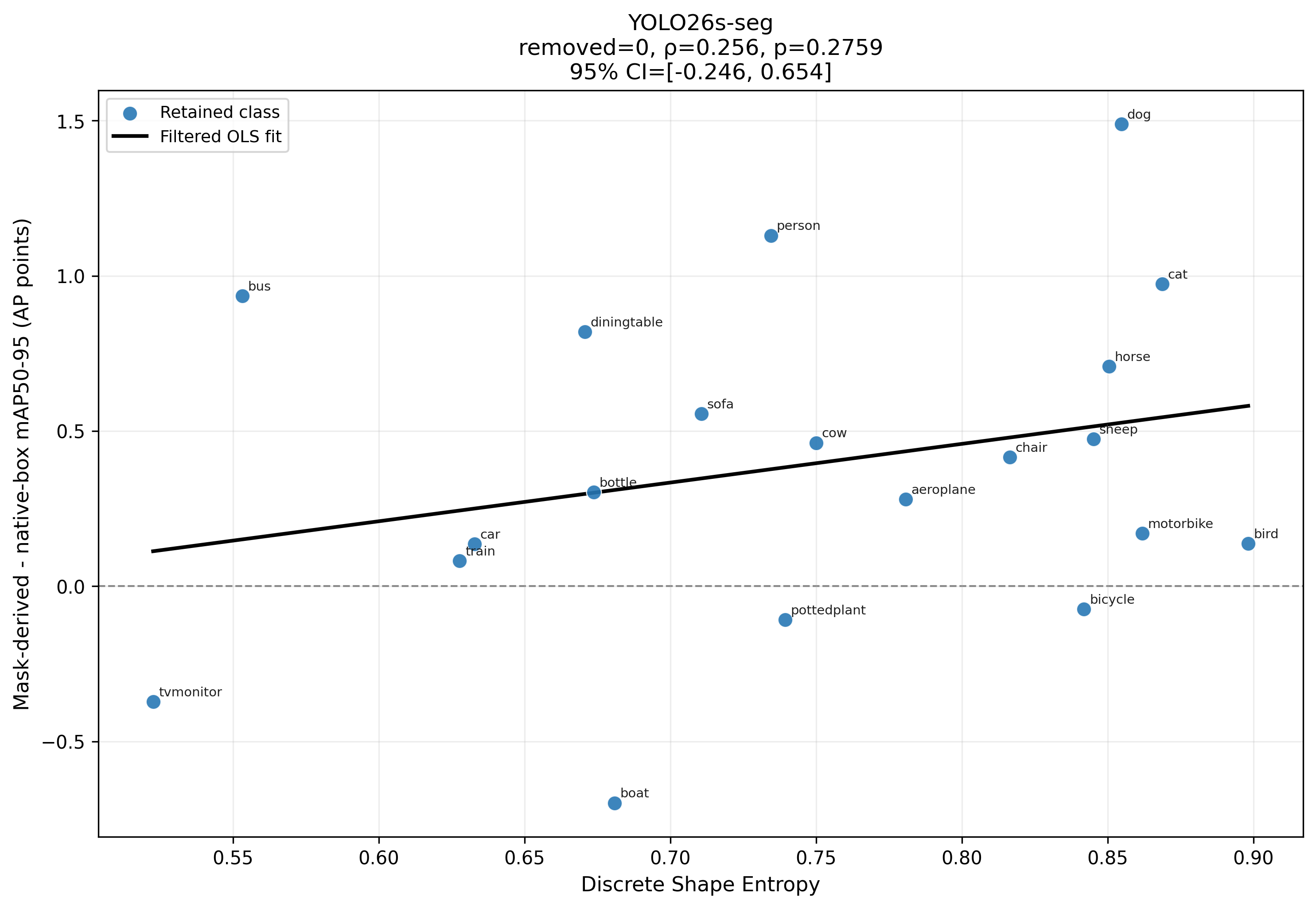}
    \caption{Correlation for YOLO26s}
\end{subfigure}
\hfill
\begin{subfigure}[t]{0.3\textwidth}
    \centering
    \includegraphics[width=\textwidth]{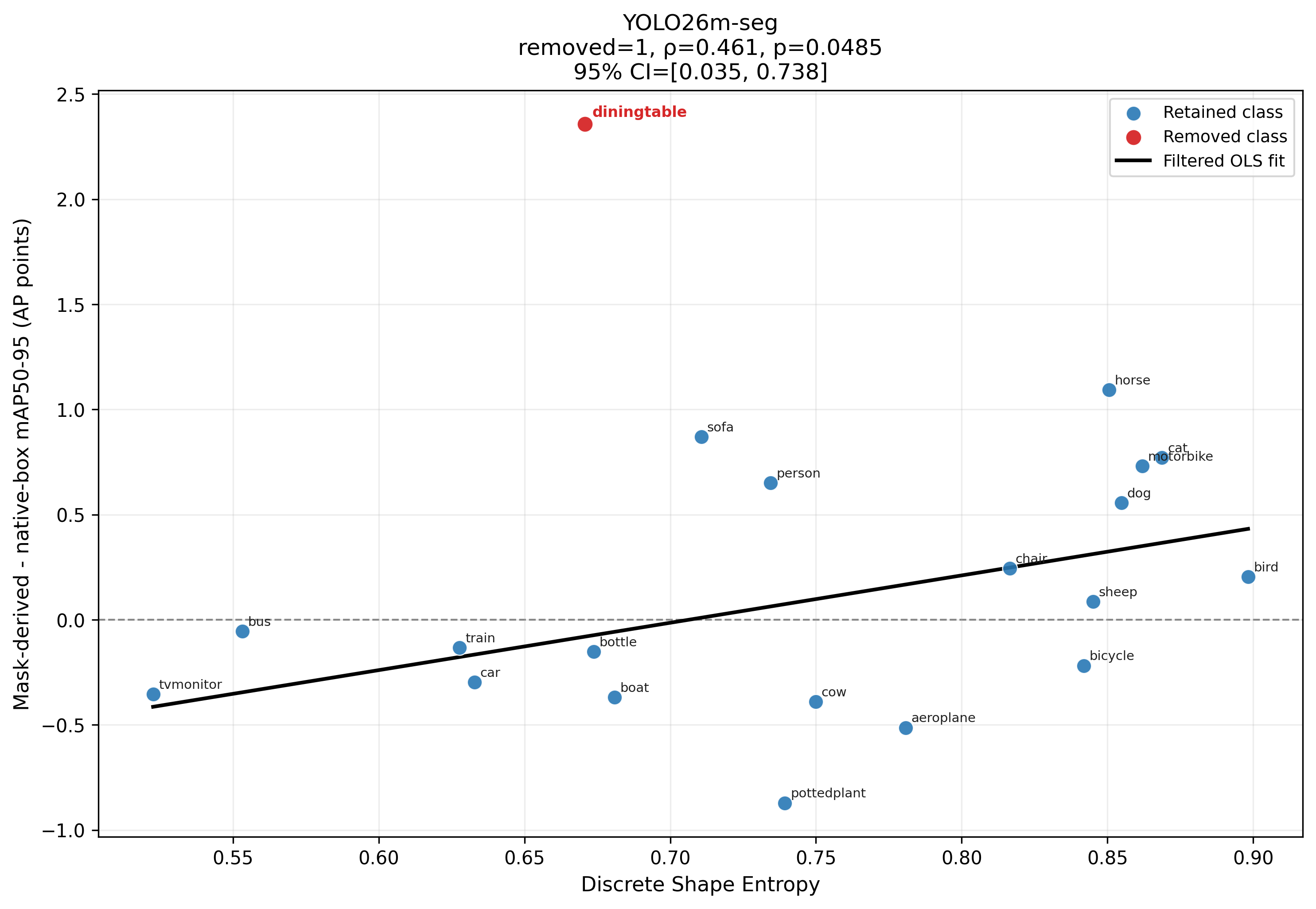}
    \caption{Correlation for YOLO26m}
\end{subfigure}
\end{figure}

\begin{figure}[h]
\ContinuedFloat
\begin{subfigure}[t]{0.3\textwidth}
    \centering
    \includegraphics[width=\textwidth]{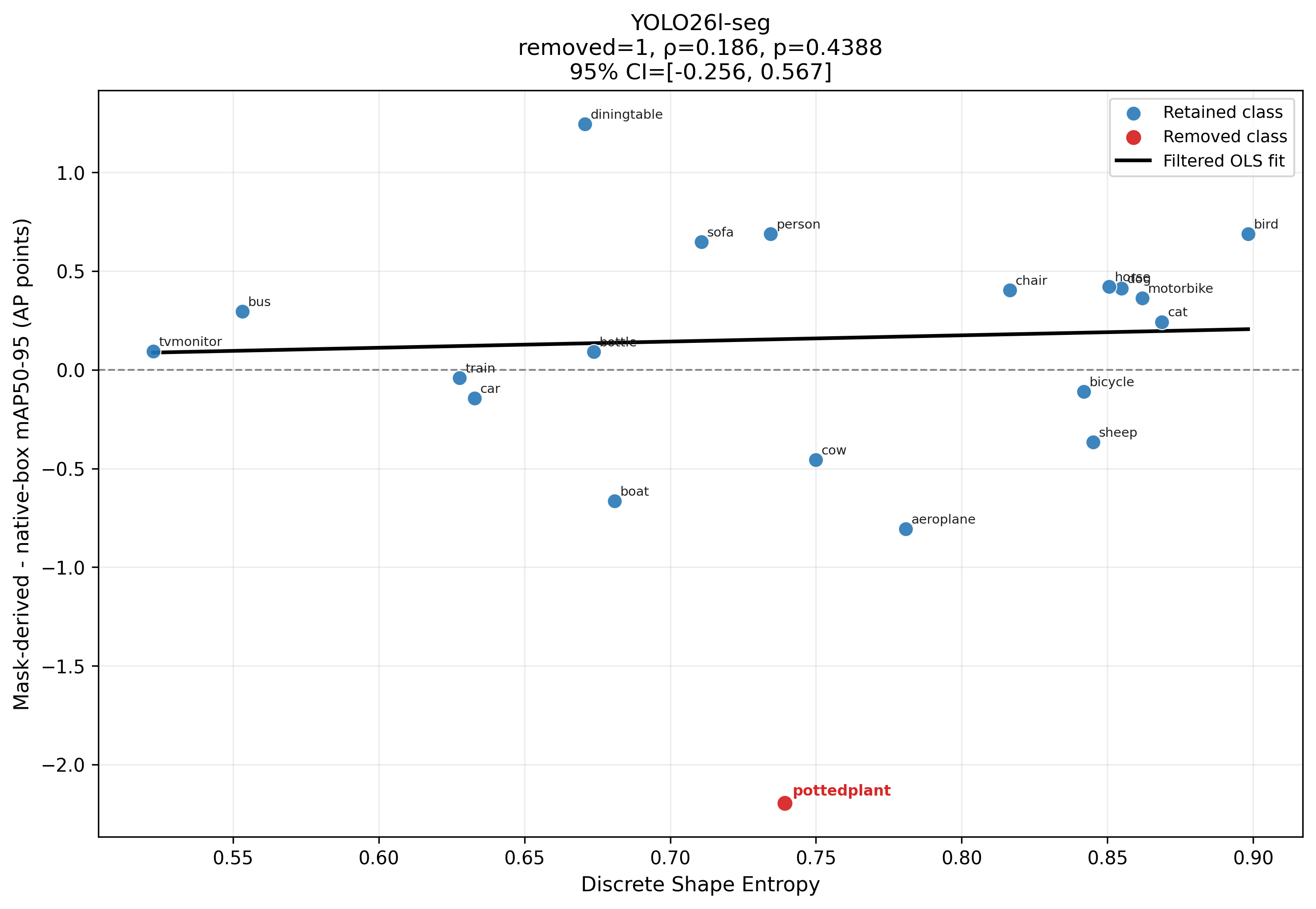}
    \caption{Correlation for YOLO26l}
\end{subfigure}
\hfill
\begin{subfigure}[t]{0.3\textwidth}
    \centering
    \includegraphics[width=\textwidth]{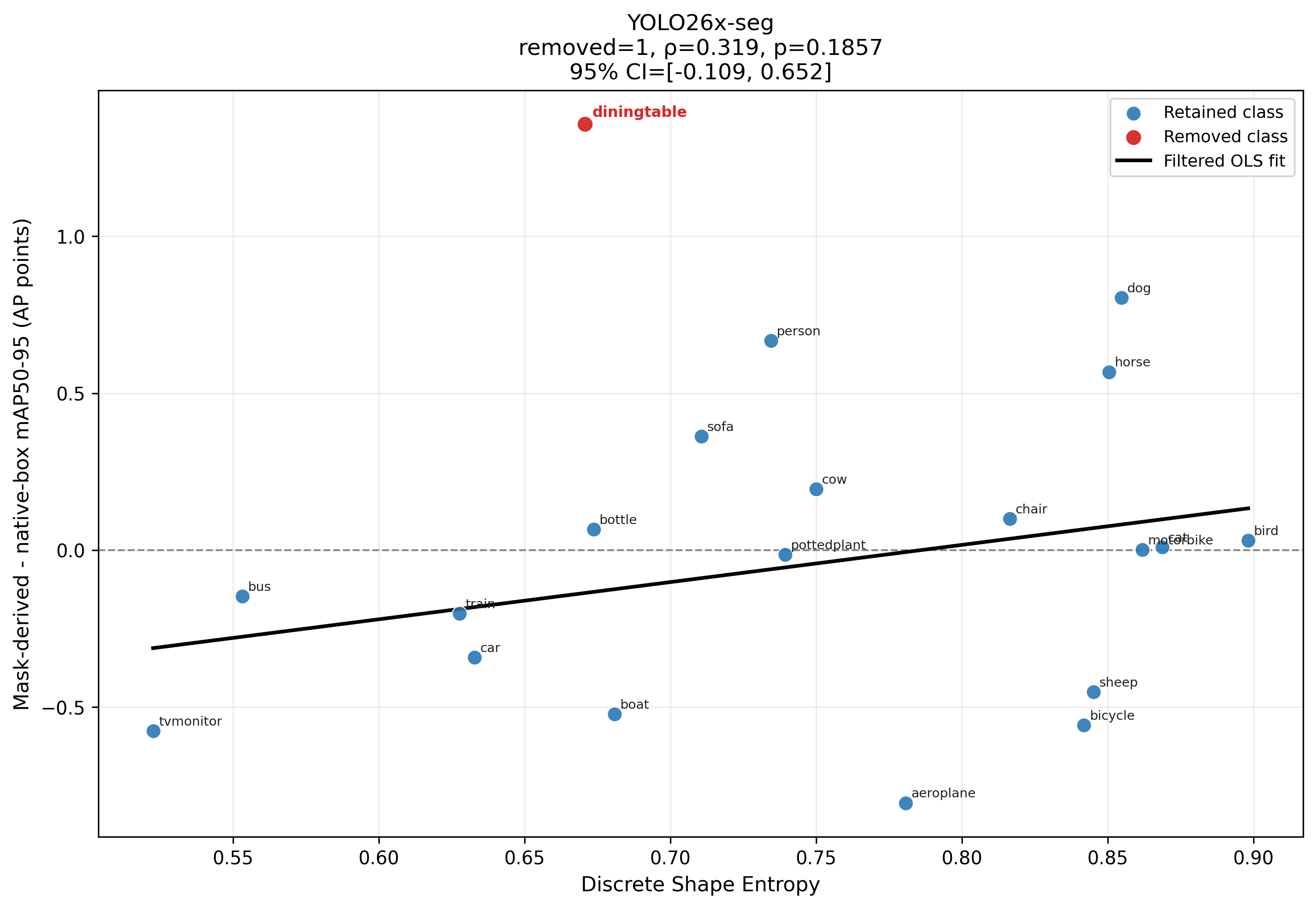}
    \caption{Correlation for YOLO26x}
\end{subfigure}
\hfill
\begin{subfigure}[t]{0.3\textwidth}
    \centering
    \includegraphics[width=\textwidth]{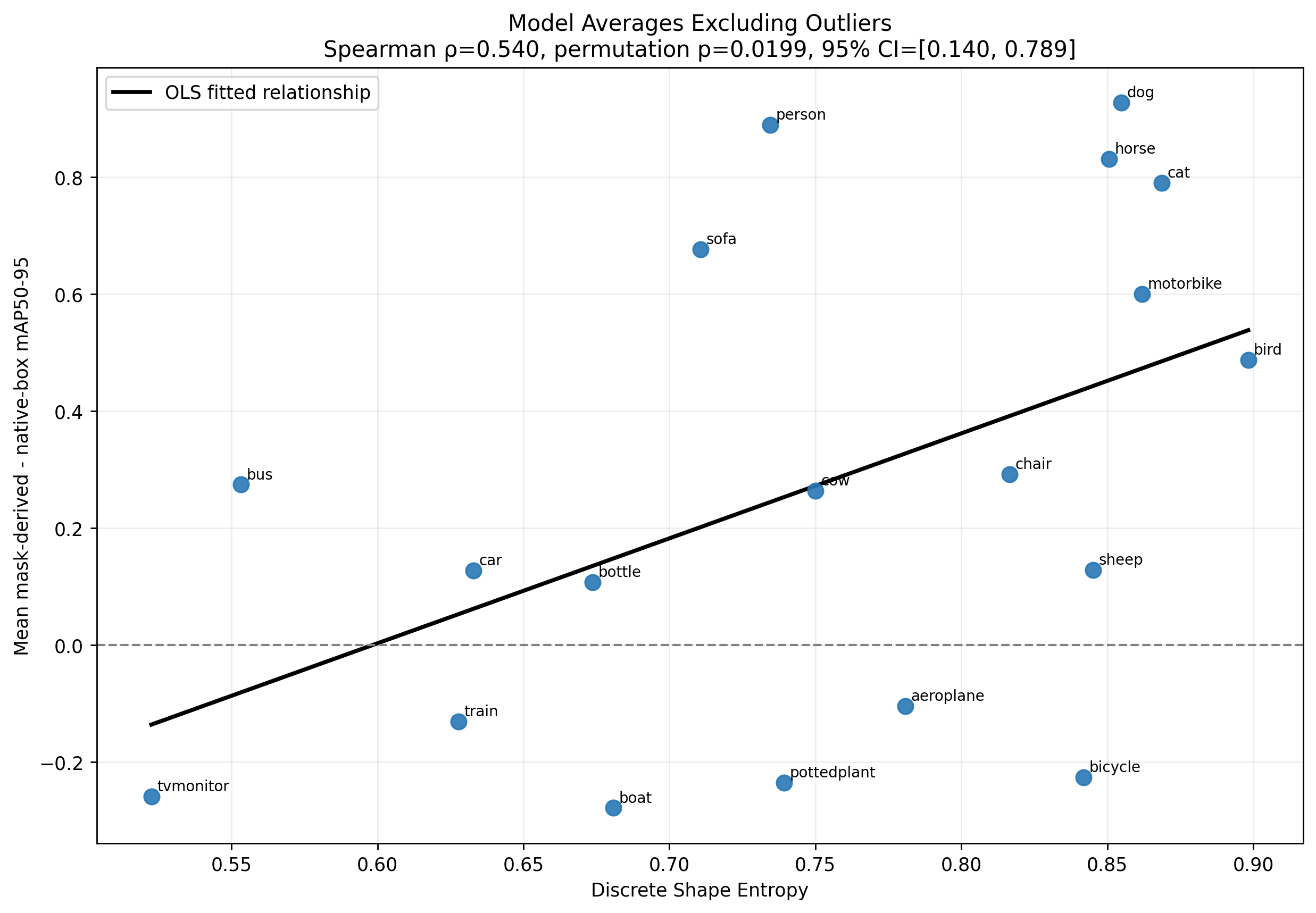}
        \caption{Model averages, no outliers}
\end{subfigure}
\end{figure}

\clearpage

\subsection{Full EAGLE Results}

Table \ref{table} reports full performance across all EAGLE variations, including various AP and computational cost metrics.

\begin{table}[h]
\caption{Full model performance with EAGLE and baseline performance.}
\label{table}
\begin{center}
\begin{tabular}{ll|lll|llll}
\toprule
Detector & Method & AP & $\text{mAP}_{50}$ & $\text{mAP}_{75}$ & Params (M) & Latency (ms)\\
\midrule
& Baseline & 40.25 & 55.80 & 43.65 & 2.4 & 2.3\\
\cmidrule{2-7}
& Enhancer & 40.89 & 56.77 & 44.17 & 2.7 & 2.3\\
\cmidrule{2-7}
& Entropy Penalty & 40.89 & 56.77 & 44.17 & 2.7 & 2.3\\
\cmidrule{2-7}
YOLO26n& Score Calibration & 41.06 & 57.07 & 44.20 & 2.7 & 2.3\\
\cmidrule{2-7}
& Geometric Bias & 40.89 & 56.76 & 44.19 & 2.7 & 2.4\\
\cmidrule{2-7}
& Duplicate Damping & 41.02 & 57.02 & \textbf{44.30} & 2.7 & 2.4\\
\cmidrule{2-7}
& Strength Adjustment & 40.89 & 56.77 & 44.17 & 2.7 & 2.4\\
\cmidrule{2-7}
& Bounded Combination & \textbf{41.11} & \textbf{57.19} & 44.25 & 2.7 & 2.3\\
\midrule
& Baseline & 52.49 & 69.69 & 57.25 & 20.4 & 6.4\\
\cmidrule{2-7}
& Enhancer & 52.76 & 70.07 & 57.50 & 20.7 & 6.4\\
\cmidrule{2-7}
& Entropy Penalty & 52.75 & 70.07 & 57.50 & 20.7 & 6.4\\
\cmidrule{2-7}
YOLO26m& Score Calibration & 52.88 & 70.37 & 57.49 & 20.7 & 6.5\\
\cmidrule{2-7}
& Geometric Bias & 52.75 & 70.06 & 57.49 & 20.7 & 6.4\\
\cmidrule{2-7}
& Duplicate Damping & 52.88 & 70.28 & \textbf{57.64} & 20.7 & 6.4\\
\cmidrule{2-7}
& Strength Adjustment & 52.76 & 70.07 & 57.50 & 20.7 & 6.5\\
\cmidrule{2-7}
& Bounded Combination & \textbf{52.94} & \textbf{70.48} & 57.56 & 20.7 & 6.4\\
\midrule
& Baseline & 48.10 & 65.07 & 52.09 & 20.0 & 6.8\\
\cmidrule{2-7}
& Enhancer & 48.16 & 65.17 & 52.17 & 20.3 & 6.9\\
\cmidrule{2-7}
& Entropy Penalty & 48.16 & 65.17 & 52.17 & 20.3 & 6.9\\
\cmidrule{2-7}
RT-DETR& Score Calibration & 48.16 & 65.17 & 52.17 & 20.3 & 6.9\\
\cmidrule{2-7}
& Geometric Bias & 48.18 & 65.16 & 52.19 & 20.3 & 6.8\\
\cmidrule{2-7}
& Duplicate Damping & 48.25 & \textbf{65.36} & 52.24 & 20.3 & 6.8\\
\cmidrule{2-7}
& Strength Adjustment & 48.17 & 65.17 & 52.17 & 20.3 & 6.9\\
\cmidrule{2-7}
& Bounded Combination & \textbf{48.26} & \textbf{65.36} & \textbf{52.26} & 20.3 & 6.9\\
\bottomrule
\end{tabular}
\end{center}
\end{table}

Models were tested across the full MS COCO 2017 validation dataset, consisting of 5000 images and 80 classes. The enhancer model represents the proposal and selector system, whereas the models with entropy penalty, score calibration, geometric bias, duplicate damping, strength adjustment, and bounded combination include the enhancer, selector, and postprocessing. All model variants except for the bounded combination do not also combine the other postprocessing steps: for example, score calibration does not also include entropy penalty.

\section{Full EAGLE Derivation}

Consider all polygon instances $\mathcal{P}$ taken from a partially supervised dataset. For each polygon instance $j$, resample to 24 equally spaced vertices, remove translation and scale, then align the resulting preshapes via generalized Procrustes alignment from this. At this step, all of the observed polygon instances are preshapes in the Kendall shape space.

For each class $c$, consider the subset of polygon instances spanning 10\% of total instances including non-supervised ones. Cluster the aligned contours based on the methodology applied in Section 4. Let $n_{c, k}$ be the count of instances of class $c$ and cluster $k$ and $n_c$ be the number of total instances of class $c$. Define the smoothed class distribution and normalized entropy as:
\begin{equation}
    p_{c, k} = \frac{n_{c, k} + \alpha}{\sum_j n_{c, j} + C \alpha}, \qquad H_c = - \frac{\sum_k p_{c, k} \log_2 p_{c, k}}{\log_2 C}, \qquad \alpha = 0.5
\end{equation}
where C denotes the number of shape clusters.

Create bins for aspect ratios based on Knuth's rule as shown in Section 4, where $r = \log \left( \frac{w}{h} \right)$ is derived based on the ground truth bounding boxes. Let $n_{c, b, k}$ represent the number of instances of class $c$, bin $b$, and cluster $k$. Let:
\begin{equation}
    p_{c, b, k} = \frac{n_{c, b, k} + \alpha + C p_{c, k}}{\sum_j n_{c, b, j} + C \alpha + C}, \qquad H_{c, b} = - \frac{\sum_k p_{c, b, k} \log_2 p_{c, b, k}}{\log_2 C}
\end{equation}

For a predicted class and box at inference, EAGLE defined prediction entropy as:
\begin{equation}
    h_i = 0.8 H_{c_i, b_i} + 0.2 H_{c_i}
\end{equation}
Now, let the frozen object detection model return feature maps $f_1, ..., f_L$ and score-ordered detections $(b_i, s_i, c_i)$. The geometric head processes the first 128 proposals. A proposal of width $w_i$ and height $h_i^{box}$ chooses feature level with $\sqrt{w_i h_i^{box}}$ with transitions at 96 and 224 pixels. A 9 by 9 feature map is generated via ROIAlign, and each level is projected across 48 channels by 1 by 1 and 3 by 3 convolutions via GroupNorm and SiLU.

For edge $e$ of the bounding box, let $f_{iek} \in \mathbb{R}^{48}$ represent feature $k$ for $k = 1, ..., 9$. We write:
\begin{equation}
    a_{iek} = \frac{\exp (g_{contact} (f_{iek}))}{\sum_{l = 1} \exp (g_{contact}(f_{iel}))}, \qquad u_{ie} = \sum_{k=1}^9 a_{iek} f_{iek}
\end{equation}
where $g_{contact}$ is the descriptor based on:
\begin{equation}
    a_{s, j} = \text{softmax}_j (W_2 \text{SiLU} (W_1 f_{s, j} + b_1) + b_2), \qquad g_{contact} = \text{concat}_{s=1}^4 \left( \sum_{j=1}^9 a_{s, j}f_{s,j} \right) \in \mathbb{R}^{192}
\end{equation}
Now, concatenate the four edge-derived $u_{ie}$ descriptors using the global ROI mean $v_i \in \mathbb{R}^{48}$ and the ratio-dependent joint value:
\begin{equation}
    g_i = \left[ \log \frac{w_i}{W}, \log \frac{h_i^{box}}{H}, \log \frac{w_i}{h_i^{box}}, \sqrt{\frac{w_i h_i^{box}}{W H}}, \text{logit}(s_i), \frac{l_i}{L - 1} \right]
\end{equation}
This provides a 246-dimensional input. To represent in 192 dimensions, it is modified via the entropy adapter:
\begin{equation}
    z_i = T_\theta ([u_{i1}, ..., u_{i4}, v_i, g_i]) + A_\phi ([h_i, h_i^2, h_i g_i])
\end{equation}
The final projection and the output are initialized to zero, therefore leaving the initial native boxes unchanged before learning geometric features.

For each box coordinate, predict a distribution over 21 support points $t_k \in [-1, 1]$, where:
\begin{equation}
    p_{iek} = \text{softmax}_k (W_e z_i), \qquad \delta_{ie} = \sum_{k=1}^{21} p_{iek} t_k
\end{equation}
where $W_e$ is taken with respect to a box edge. Let $D_i = (w_i, h_i^{box}, w_i, h_i^{box})$ be the original proposal that will be refined. Let the adapted box be:
\begin{equation}
    b_i^p = \text{clip} [b_i + 0.2 G (h_i) D_i \odot \delta_i], \qquad G(h) = 0.6 + 0.4 \sigma(4h -2)
\end{equation}

Each proposal is paired with its highest-IoU same-class ground truth box as determined via fitting on training data.

For supervised proposal $i$ with box $b_i^\star$, define target displacement as:
\begin{equation}
    \delta_i^\star = \text{clip}_{[-1, 1]} \frac{b_i^\star - b_i}{0.2 G (h_i) D_i}
\end{equation}
The loss linearly interpolates cross-entropy between the two support points, combining each component of $\delta_i^\star$. Localization value $1 - \text{CIoU} (b_i^p, b_i^\star)$ is combined with normalized coordinate Smooth L1 error. This penalizes the geometric head for not being able to improve native IoU. Loss is averaged with normalized weights:
\begin{equation}
    \hat{w}_i \propto \text{clip} \left[ \exp \left( 0.25 \frac{h_i - \hat{h}}{\max (sd(h), 0.05)} \right), 0.65, 1.35 \right]
\end{equation}
For selected polygon, we define the target distribution at each box side by sampling 9 points and determining squared distance $d^2_{iek}$ to the polygon boundary. This is done via:
\begin{equation}
    r_{iek} = \frac{\exp [-d^2_{iek} / (2 \sigma^2)]}{\sum_l \exp [-d^2_{iel} / (2 \sigma^2)]}, \qquad \sigma = 0.08
\end{equation}
Together, cross entropy with $r_{ie}$ and attention logits supplies $L_{contact}$. This applies only when analyzing predictions from the polygon-supervised instances.

During quality training, assign the ground-truth object to its best refined candidate. We may write the separate stages of training, where the first 25 epochs are geometric training, the next 20 are quality training, and the last 15 are joint, via the following loss functions:
\begin{equation}
    \mathcal{L}_{geom} = \mathcal{L}_{loc} + 0.45 \mathcal{L}_{dist} + 0.1 \mathcal{L}_{margin} + 0.1 \mathcal{L}_{contact}
\end{equation}
\begin{equation}
    \mathcal{L}_{quality} = 0.8 \mathcal{L}_{APquality} + 0.5 \mathcal{L}_{gain} + 0.2 \mathcal{L}_{rank}
\end{equation}
\begin{equation}
    \mathcal{L}_{joint} = 0.2 \mathcal{L}_{geom} + 0.5 \mathcal{L}_{quality}
\end{equation}
These methods are used to optimize the main geometric head, while the method for training the selector and fitting postprocessing are different.

The selector is trained on cached outputs from the detector and EAGLE enhancement, where there is a single detector pass across the entire training. The selector evaluates 64 detections per image, while the geometric head processes up to 128 proposals.

For detection $i$, let $b_i$ be the native box and $b_i^p$ be the geometric proposal, as well as let $\Delta_i$ be the enhancement's predicted gain. The existing soft refinement:
\begin{equation}
    b_i^{(0)} = b_i + 0.75 \sigma \left(\frac{\Delta_i}{0.015}\right)(b_i^p - b_i)
\end{equation}
is the default soft refinement. The alternatives being selected between are the native box, and refinements at 0.5, 0.75, 1 and 1.25 times the proposed displacement of each edge. The selector learns how to adjust the geometric refinement which is proposed with entropy-aware changes to the underlying predictions, therefore choosing between proposals to adapt the native object detection and geometric heads together to achieve higher accuracy. For each alternative, training replaces the detection's box and rematches the affected class's detections against ground truth at 10 IoU thresholds. The resulting target is a first-order AP estimation, from which the highest AP value is selected between the proposals.

The network projects the geometric head's 192-dimensional representation alongside 50 contextual features separately to 64 dimensions and then appends an eight-dimensional class embedding, applying a 136 to 96 to 64 MLP. The output is $z_i$ and $h_i$ is its shape entropy. We utilize $\hat{z}_i = z_i \odot [ 1 + 0.1 \tanh M(h_i, h_i^2) ]$ to include entropy as additional information from which the MLP is able to better interpret entropy values alongside the proposals themselves. Together, the selector uses three branches to predict utility and likelihood of the change improving detection AP across the alternatives, with localization quality tested at 10 IoU thresholds across the possible adaptations to the box that can be made. Training minimizes the joint loss function derived as:
\begin{equation}
    \mathcal{L}_{selection} = \mathcal{L}_{utility} + 0.15 \mathcal{L}_{harm} + 0.5 \mathcal{L}_{quality} + 0.15 \mathcal{L}_{action}
\end{equation}
where each uses weighted Smooth L1 regression, cross-entropy, masked quality binary cross-entropy, and ranking cross-entropy, respectively.

During inference, the selector maximizes the predicted utility and includes a penalty for predicted harm of the refinement. If $\hat{q}_i$ is the mean predicted quality, then rescoring provides:
\begin{equation}
    s'_i = \sigma [ \text{logit} (s_i) + \beta \text{clip} (\text{logit}(\hat{q}_i) - \text{logit}(s_i), -1.5, 1.5 ) ]
\end{equation}
where the hyperparameter $\beta$ determines how much the predicted quality is able to rescore the refinement selection.

For postprocessing, again consider $b_i^\star$ as the EAGLE enhanced box and $s_i'$ as the score, as well as $b_i$ as the native box, $h_i$ as the native box entropy, and $x_i$ as the calibration features. The postprocessing operations are performed as follows:
\begin{enumerate}
    \item Entropy Penalty: $s_i^E = s_i' \exp( - \lambda_E h_i)$.
    \item Score Calibration: $s_i^C = \sigma (\text{logit} (s_i') + \lambda_C \text{clip} (x_i^T \theta, -1.5, 1.5))$.
    \item Geometric Bias: For box center of $c_i$ along with size $w_i$, let $d_i = x_i^T W$ where $W$ is fitted to box edges. Then, we have $c_i^B = c_i + \lambda_B w_i \odot \text{clip} (d_{i,1:2}, -0.02, 0.02)$ and $w_i^B = w_i \odot \exp [\lambda_B \text{clip} (d_{i,3:4}, -0.04, 0.04)]$.
    \item Duplicate Damping: Let $u_i$ be the greatest IoU with an earlier and higher scoring same-class detection, then $s_i^D = s_i' \exp \{- \lambda_D \left[\frac{(u_i - t)_+}{1 - t}\right]^2 (0.5 + 0.5 h_i)\}$.
    \item Refinement Strength: $b_i^S = b_i + \text{clip} [1 + \lambda_S (h_i - 0.5), 0, 1.5] (b_i^\star - b_i)$.
\end{enumerate}
The bounded combination tests the combinations of either entropy penalty or score calibration followed by one of geometric bias, duplicate damping, or refinement-strength adjustment, then chooses the sequence with the highest MS COCO 2017 training AP and uses that combination for validation. The coefficients used were $\lambda_E = 0.7, \lambda_C = 1, \lambda_B = 1, \lambda_D = 0.75, \lambda_S = 1$, along with an alternative test of $\lambda_S = -1$, of which the higher performing step is used for validation.

\end{document}